\documentclass[conference]{IEEEtran}
\IEEEoverridecommandlockouts
\usepackage{cite}
\usepackage{amsmath,amssymb,amsfonts}
\usepackage{algorithmic}
\usepackage{graphicx}
\usepackage{textcomp}
\usepackage{xcolor}
\usepackage{algorithmic}
\usepackage{booktabs} 
\usepackage[linesnumbered,lined,ruled,titlenumbered]{algorithm2e}

\usepackage{epsfig}
\usepackage{amssymb}
\usepackage{amsthm}
\usepackage{graphicx}
\usepackage{textcomp}
\usepackage{xcolor}
\usepackage{booktabs}
\usepackage{multirow}
\usepackage{subfigure}
\usepackage{color}
\usepackage{lipsum}
\usepackage[misc]{ifsym}
\usepackage{amsmath}
\usepackage{setspace}
\usepackage{pifont}
\usepackage{bbding}
\usepackage{tabularx}
\usepackage{url}
\usepackage{threeparttable}
\usepackage{algorithmic}
\usepackage[linesnumbered,lined,ruled,titlenumbered]{algorithm2e}
\usepackage{xcolor}

\theoremstyle{thmstyleone}

\newtheorem{theorem}{Theorem}

\theoremstyle{thmstyletwo}

\theoremstyle{thmstylethree}
\newtheorem{definition}{Definition}

\def\BibTeX{{\rm B\kern-.05em{\sc i\kern-.025em b}\kern-.08em
    T\kern-.1667em\lower.7ex\hbox{E}\kern-.125emX}}
\begin{document}

\title{A Clustering Method with Graph\\Maximum Decoding Information}

% \author{Anonymous IJCNN Submission}
\author{
\IEEEauthorblockN{Xinrun Xu\textsuperscript{1,2}$^\text{*}$, Manying Lv\textsuperscript{1,2}$^\text{*}$, Zhanbiao Lian\textsuperscript{1,2}, 
Yurong Wu\textsuperscript{1,2}, 
Jin Yan\textsuperscript{1,2}$^\text{\dag}$,
Shan Jiang\textsuperscript{3}$^\text{\dag}$, Zhiming Ding\textsuperscript{1,2}$^\text{\dag}$}
\IEEEauthorblockA{
\textsuperscript{1}{University of Chinese Academy of Sciences}, Beijing, China,\\
\textsuperscript{2}{Institute of Software, Chinese Academy of Sciences}, Beijing, China \\
\textsuperscript{3}{Advanced Institute of Big Data (Beijing), Beijing, China}
}
Email:\{xuxinrun20, lvmanying21,
yvette.yan\}@mails.ucas.ac.cn,
jiangshan@alumni.nudt.edu.cn,
zhiming@iscas.ac.cn
}
\maketitle
\renewcommand{\thefootnote}{\fnsymbol{footnote}}
\footnotetext[1]{Equal Contribution.} 
\footnotetext[2]{Corresponding Author.}
\footnotetext[3]{This work was supported by the National Key R\&D Program of China (No. 2022YFF0503900).}
\renewcommand{\thefootnote}{\arabic{footnote}}

\maketitle

\begin{abstract}
The clustering method based on graph models has garnered increased attention for its widespread applicability across various knowledge domains. Its adaptability to integrate seamlessly with other relevant applications endows the graph model-based clustering analysis with the ability to robustly extract "natural associations" or "graph structures" within datasets, facilitating the modelling of relationships between data points. Despite its efficacy, the current clustering method utilizing the graph-based model overlooks the uncertainty associated with random walk access between nodes and the embedded structural information in the data. To address this gap, we present a novel \textbf{C}lustering method for \textbf{M}aximizing \textbf{D}ecoding \textbf{I}nformation within graph-based models, named CMDI. CMDI innovatively incorporates two-dimensional structural information theory into the clustering process, consisting of two phases: graph structure extraction and graph vertex partitioning. Within CMDI, graph partitioning is reformulated as an abstract clustering problem, leveraging maximum decoding information to minimize uncertainty associated with random visits to vertices. Empirical evaluations on three real-world datasets demonstrate that CMDI outperforms classical baseline methods, exhibiting a superior decoding information ratio (DI-R). Furthermore, CMDI showcases heightened efficiency, particularly when considering prior knowledge (PK). These findings underscore the effectiveness of CMDI in enhancing decoding information quality and computational efficiency, positioning it as a valuable tool in graph-based clustering analyses.
\end{abstract}

\begin{IEEEkeywords}
Clustering, Decoding Information, Structure Extraction, Graph Partitioning
\end{IEEEkeywords}

\section{Introduction}
Clustering, a crucial unsupervised learning method, organizes $n$ data points into $k$ clusters ($k \textless n$) based on characteristics like distance or similarity. Used in various domains such as market analysis, biological classification, and geographic data analysis, it helps identify natural associations in data and serves as a pre-processing step for other data mining techniques. Our study focuses on clustering for discovering natural data associations.

Researchers have proposed several clustering models, including connectivity models \cite{DBLP:conf/kdd/MonathDGZAMMNTT21}, centroid models \cite{chen2023optimization}, density models \cite{DBLP:journals/pr/ChenZBWCD21}, subspace models \cite{DBLP:journals/jcst/WangZCQZLZZ21,ahmed2021remaxint,DBLP:journals/sac/AffeldtLN21}, and graph-based models \cite{DBLP:journals/isci/YangDLLGUH21,DBLP:journals/concurrency/MaLZLC21,jiang2022flow}. For example, hierarchical clustering \cite{DBLP:conf/kdd/MonathDGZAMMNTT21,DBLP:journals/kbs/NguyenASCA20} uses connectivity, while K-Means \cite{bezdan2021hybrid,chen2023optimization} and DBSCAN \cite{DBLP:journals/pr/ChenZBWCD21,an2023strp} are based on centroid and density models, respectively. Subspace models include bi-clustering \cite{DBLP:journals/jcst/WangZCQZLZZ21} and co-clustering \cite{ahmed2021remaxint}, while spectral clustering \cite{DBLP:conf/aaai/SunCWLF20} exemplifies the graph-based model, applied in fields from mutation testing \cite{DBLP:journals/infsof/WeiYGL21} to speech separation \cite{zhang2023single}.

Spectral clustering uses the similarity matrix's eigenvalues for dimensionality reduction and cluster formation. It identifies natural associations in data by optimizing edge weights in graph structures through graph cutting. This paper is inspired by the spectral clustering approach within the graph model.

Graph-based data clustering, especially spectral clustering, faces challenges in handling large and high-dimensional datasets due to the complexity of solving the Laplacian matrix's eigenvalues and eigenvectors. Additionally, effective graph cut or vertex partitioning is essential but difficult to achieve. Existing methods like RatioCut \cite{DBLP:conf/icassp/LiNL18} and Ncut \cite{DBLP:conf/iccv/ChewC15} focus on optimizing graph weights but often neglect the uncertainty of node connectivity and the structural information in data space \cite{DBLP:journals/tit/LiP16,DBLP:journals/tit/LiYXWHWDXZ18}.

To address these issues, we introduce CMDI, a new algorithm for graph-based clustering that leverages two-dimensional structural information theory \cite{DBLP:journals/tit/LiP16,xu2023optimal} to partition graph vertices. CMDI aims to decode the intrinsic structure of systems and maximize system decoding information. Using structural entropy theory, it divides vertices into clusters for effective decoding of dataset structures. CMDI operates under two scenarios: one where natural data associations are known, forming a graph with data as vertices and relationships as edges, and another where inherent relationships are discerned from real-world data. CMDI optimizes node partitioning based on maximizing decoding information, derived from the difference between one-dimensional and two-dimensional structural entropy \cite{DBLP:journals/tit/LiYXWHWDXZ18, DBLP:journals/tit/LiP16}. This approach is crucial for establishing robust data relationships, which will be detailed further in Section \ref{se:02}.

This paper makes several significant contributions to the field of data clustering:
\begin{itemize}
\item 
Introduction of CMDI, a novel graph-based clustering method, unique in its integration of two-dimensional structural information theory. This approach maximizes decoding information, a first in clustering methodologies.
\item Development of a strategy to maximize decoding information for CMDI, which aids in establishing relationships between data points and extracting natural associations. Additionally, the paper introduces CMDI-PKs, which incorporate various prior knowledge mechanisms to improve CMDI's performance.

\item Systematic exploration and evaluation of different techniques for extracting structural information from data points. This includes the study of diverse proximity metrics to understand their effectiveness comprehensively.

\item Extensive experimental testing on three large-scale real-world datasets, demonstrating CMDI and CMDI-PK's superior performance over popular baselines. Notably, CMDI, particularly with the inclusion of prior knowledge, shows enhanced efficiency and better decoding information ratios.
\end{itemize}

\section{Structural Information And\\Decoding Information}\label{se:02}

\subsection{Essential Definitions}
Here we first give some basic definitions as follows:

\begin{definition}[One-dimensional structural information (ODSI).]
    The one-dimensional structural information of an undirected graph $G(V,E,W)$ is defined as 
    $$
    H_{1}(G)=\sum^{n=\mid V\mid}_{i=1} f \left( \operatorname{deg} \left( v_{i} \right) / \operatorname{vol}(G) \right)
    $$ 
    where $v_i\in V$ is the vertex of $G$, $de\textsl g(v_i)$ is the degree of a vertex $v_i$, and $\operatorname{vol}(G)=\sum_{i=1}^{n=\mid V \mid} \operatorname{de\textsl g}\left(v_{i}\right)$ is the volume of $G$. Usually, the mathematical expression of $f(^*)$ is defined as $f(x)=-x \log _{2} x$ in literature, where $x\in [0,1]$ and $f(0)=f(1)=0$. For the graph with empty edge sets, $vol(G)=0$.
    \label{de:01}
\end{definition}

\begin{definition}[Two-dimensional structural information (TDSI).]
    Given set $P=\left\{V_{1}, V_{2}, \cdots, V_{L}\right\}$, where $\vert P\vert=\vert V\vert$, $V_i\subset V$, and $\cap V_i=\varnothing$, $\cup V_i=V$. Then the two-dimensional structural information \cite{DBLP:journals/tit/LiP16} of an undirected graph $G(V,E,W)$ is defined as $
    H_{2}^{p}(G)=\sum_{j=1}^{L}[\frac{\operatorname{vol}\left(V_{j}\right)}{\operatorname{vol}(G)}\cdot\sum_{i=1}^{N_{i}}f\left(\operatorname{de\textsl g}\left(v_{i}^{j}\right) /\operatorname{vol}\left(V_{j}\right)\right)+
    \frac{\textsl g_{j}}{\operatorname{vol}\left(V_{j}\right)} \cdot f\left(\operatorname{vol}\left(V_{j}\right) /\operatorname{vol}(G)\right)]
    $
    
    where $\textsl g_i$ is the number of edges from $V_i$ to nodes outside of $V_i$.
    \label{de:02}
\end{definition}

\begin{definition}[Optimal two-dimensional structural information (OTDSI).]
    The optimal two-dimensional structural information of an undirected graph $G(V,E,W)$ is defined as $H_{2}(G)=\min _{P}\left\{H_{2}^{P}(G)\right\}$ where $P$ runs over all the partitions of the $G$. The optimal coding tree is a priority tree and is a hierarchical cluster or a partition sets that minimize the uncertainty of the graph.
    \label{de:03}
\end{definition}

\begin{definition}[Decoding information (DI).]
    The decoding information of an undirected graph
$G(V,E,W)$ is defined as $$D(G)=H_{1}(G)-H_{2}(G)$$ where $H_1(G)$ is the ODSI of the $G$, and $H_2(G)$ is the OTDSI of the $G$.
    \label{de:04}
\end{definition}

\subsection{Statements on Decoding Information}
This paper focuses on the concept of two-dimensional structural information entropy, which measures the uncertainty in a graph under an optimal coding tree, representing the graph's inherent uncertainty at a more complex level than one-dimensional structural entropy. The optimal coding tree is a hierarchical model and data structure abstracting the graph, with the 2D structural information entropy dependent on the uncertainty within this tree. Decoded information, a key concept in this study, is calculated by subtracting the one-dimensional structural entropy from the two-dimensional structural information entropy.

Decoding information quantifies the uncertainty reduction achievable through hierarchical abstraction in a graph. Maximizing this decoding information is crucial for uncovering the intrinsic structure of system components or graphs. The process of maximizing decoding information is broken down into two sub-problems: extracting structure from the data space and identifying the partition set that minimizes the two-dimensional structural information entropy. This approach presents a new principle for exploring the substantial structure hidden within the data space.

\section{Extracting Structure}\label{se:03}

This paper addresses challenges in constructing structures from data points, focusing on two primary difficulties. First, it involves extracting intrinsic relationships among system components when only the number of components is known. Second, in many real-world systems, the exact number of components is often unknown, resembling "black-box" systems, where only sample data points are available without additional context. This situation is common in datasets provided for analysis and mining, making the extraction of substantial structure from these datasets crucial for insightful data analysis.

\subsection{Extracting Structure Based on Mutual Information}

A key method for structure extraction discussed is Mutual Information (MI) \cite{DBLP:journals/tit/LiA21b}, based on information theory. MI is used for quantifying the information that one random variable contains about another, effectively calculating statistical similarities between node pairs. This method is capable of detecting both linear and nonlinear relationships between two random variables, $\Gamma_{u}$ and $\Gamma_{v}$. The estimation of MI forms a part of the process for understanding and analyzing the underlying relationships in data.

\begin{equation}
  \begin{aligned} 
    M I_{u v} &=H\left(\Gamma_{u}\right)-H\left(\Gamma_{u} \mid \Gamma_{v}\right)=\sum_{u, v} p(u, v) \log \frac{p(u, v)}{p(u) p(v)} \\
     &=\sum_{u, v} p_{u v} \log \frac{p_{u v}}{p_{u} p_{v}} 
  \end{aligned}
  \label{eq:01}
\end{equation}
where both $p_u$ and $p_v$ are the probability density function (PDF) of the vertex $u$,
$p_{uv}$ is the joint PDF of a pair $(u,v)$.
$H\left(\Gamma_{u}\right)$ is the entropy of $\Gamma_{u}$, and
$H\left(\Gamma_{u} \mid \Gamma_{v}\right)$ is the conditional entropy of
$\Gamma_{u}$ given $\Gamma_{v}$. Notably, the matrix $MI$ is symmetric.

\subsection{Extracting Structure via Maximum Likelihood Estimation}

\begin{algorithm}[htbp]
  \caption{\textbf{Extracting Structure Method Based on MLE, ESMBMLE}}
  \label{algo:02}
  \KwIn{$\sigma(t)\leftarrow\textbf{X}\in R^{n\times T}$ - $n-$ dimensional time-series data $\sigma(t)$ with time length $T$.\\ 
  \textbf{Initialize:}  $\alpha$ - The learning rate parameter.}
  \KwOut{ $G$ - Graph.}
  $W\leftarrow InitRandom\{(n, n)\}$.  // $W \in R^{n\times n}$,  initialize  $W_{ij}$ at random.\\
  \For{$i$ in $Ran\textsl ge(n)$}{
    $H_i(\sigma(t)) \leftarrow \sum_jW_{ij}\sigma_j(t)$.\\\label{algo02ln03}
    $P[\sigma_i(t+1)|\sigma(t)] \leftarrow \frac{exp[\sigma_i(t+1)\cdot H_i(\sigma(t))]}{exp[H_i(\sigma(t))]+exp[-H_i(\sigma(t))]}$.\\\label{algo02ln04}
    $P_i \leftarrow \prod_{t=1}^{T-1}P[\sigma_i(t+1)|\sigma(t)]$.\\\label{algo02ln05}
    $W_{ij} \leftarrow W_{ij}+\frac{\alpha}{T-1}\frac{\partial ln P_i}{\partial W_{ij}}$ with Equation \eqref{eq:05}. 
  }
  $A \leftarrow Threshold(W)$.  // $A \in R^{n\times n}$ is the adjacency matrix.\\
  $G \leftarrow CreateGraph(A)$.  // Create $G$ from $A$.\\
  \Return{G}.\label{algo02ln10}
\end{algorithm}

Maximum Likelihood Estimation (MLE) \cite{DBLP:journals/csda/YuTR20} is a widely adopted statistical method effective for parameter inference in stochastic models. MLE is particularly advantageous in large datasets, converging to the true parameters of the system model. To unveil the intrinsic structure of a system, it becomes imperative to observe the state of each component within this "black box" system, allowing for the comprehensive capture of the dynamical process and, subsequently, the extraction of essential structure. 
In this study, we employ the Kinetic Ising model (KIM) \cite{DBLP:journals/amc/FonsecaKMM15} to simulate the dynamic process of the graph. In KIM, the $n$-spin state $\sigma=\{\sigma_{1}, \sigma_{2}, \cdots, \sigma_{n}\}$ at time $t+1$ is stochastically determined from the current state $\sigma(t)$ at time $t$ with the following probability:

\begin{equation}
  P\left[\sigma_{i}(t+1) \mid \sigma(t)\right]\!=\!\frac{\exp \left[\sigma_{i}(t+1) \cdot H_{i}(\sigma(t))\right]}{\exp \left[H_{i}(\sigma(t))\right]+\exp \left[-H_{i}(\sigma(t))\right]}
  \label{eq:02}
\end{equation} 
where $H_{i}(\sigma(t))=\sum_{j} W_{i j} \sigma_{j}(t)$ represents the influence of the present state
$\sigma(t)$ on the future state
$\sigma_i(t+1)$ for $i=1,2,\cdots,n$. $W_{ij}$ is the weight matrix.

Given $n-$dimensional time-series data $\sigma(t)$ with length $T$, the data likelihood is defined as follows:
\begin{equation}
  \setlength{\abovedisplayskip}{1pt}
  \setlength{\belowdisplayskip}{1pt}
  \mathrm{P}=\prod_{t=1}^{T-1} \prod_{i=1}^{n} P\left[\sigma_{i}(t+1) \mid \sigma(t)\right]
  \label{eq:03}
\end{equation}

Using MLE, one can optimize $W_{ij}$ to increase $\ln\mathrm{P}$ as follows:
\begin{equation}
  \setlength{\abovedisplayskip}{1pt}
  \setlength{\belowdisplayskip}{1pt}
  {W}_{i j}^{\prime}=W_{i j}+\frac{\alpha}{T-1} \frac{\partial \ln \mathrm{P}}{\partial W_{i j}}
  \label{eq:04}
\end{equation}
where $\alpha$ represents the learning rate with the gradient ascent method. Based on Equation 
\eqref{eq:02}, the solution of the gradient $\frac{\partial \ln \mathrm{P}}{\partial W_{i j}}$ is defined as follows:
\begin{equation}
  \setlength{\abovedisplayskip}{1pt}
  \setlength{\belowdisplayskip}{1pt}
  \frac{\partial \ln \mathrm{P}}{\partial W_{i j}}\!=\!\sum_{t=1}^{T-1}\left\{\sigma_{i}(t\!+\!1)\! \cdot\! \sigma(t)\!-\!\tanh \left[H_{i}(\sigma(t))\right]\! \cdot\! \sigma_{j}(t)\right\}
  \label{eq:05}
\end{equation}
where $\tanh (x)=\frac{\sinh x}{\cosh x}$.

\subsection{Extracting Structure Based on Proximity Metrics}
ESMBMI and ESMBMLE Algorithms are effective for reconstructing systems with known components but unclear relationships. However, they falter when component information is also unclear, as in "black box" systems. In such cases, using proximity metrics to create a graph from sensor-sampled data, characterized by proximities $w_{i,j}$, is crucial for revealing local neighbourhood relationships and reconstructing the system's structure.

\textbf{Proximity Metrics.}
Proximity metrics for numeric variables primarily encompass Euclidean (EUC), Manhattan (MANH), Minkowski (MINK), Chebyshev (CHEB), Canberra \cite{DBLP:journals/ijisp/MahantaK18} (CANB), Mahalanobis (MAHA), Angular (ANGU), Pearson Correlation \cite{DBLP:journals/ijcse/ShahZ21} (P-COR), Gaussian Similarity (GAUS), $\varepsilon$-Neighborhood ($\varepsilon$-NE), and $k$-Nearest Neighbors ($k$-NN). The CANB metric computes absolute distance, yet it is not scale-dependent. Conversely, the MAHA, ANGU, and P-COR measures predominantly assess the correlation between variables and exhibit scale invariance.

\textbf{Extracting Structure.}
Initiating with the Gaussian kernel function \cite{DBLP:journals/cin/ZhangLW20} (also known as the radial basis function, RBF), the proximity metrics $d_{x,y}$ contribute to the construction of an affinity matrix. Subsequently, the adjacency matrix of the neighbourhood graph is naturally derived by applying a suitable threshold to the constructed affinity matrix.

\begin{equation}
  \setlength{\abovedisplayskip}{1pt}
  \setlength{\belowdisplayskip}{1pt}
  A_{x, y}=\exp \left(-d_{x, y}^{2} /\left(2 \sigma^{2}\right)\right)
  \label{eq:06}
\end{equation}

$\sigma$ in the Gaussian kernel defines its width, and $d_{x,y}$ measures similarity between data points $x$ and $y$. In this model, local neighbours are linked with high weights ($A_{x, y}$), while distant connections have low, positive weights, leading to a non-sparse adjacency matrix.

\textbf{Structure Measures:}
Assessing similarity metrics for effective neighbourhood graph construction is challenging, but evaluating the graphs helps measure their quality. Decoding information (DI) is used to assess these graphs; higher DI indicates more information in the graph, showing the effectiveness of the similarity measurement method. The choice of similarity/proximity function depends on the DI in the structure, guiding data point set structure extraction.

\section{CMDI: Clustering Algorithm for\\Maximum Decoding Information}\label{se:04}
A graph representing a system's sampled data points is derived using a specific extraction method. The graph's decoding information (DI) is key to analyzing it within the dataset. DI aims to find the optimal node partition, minimizing two-dimensional structural information entropy. Currently, there's no direct method for obtaining the optimal two-dimensional structure entropy. However, the study proposes a greedy algorithm to approximate this optimal partition, effectively splitting nodes into distinct clusters and maximizing DI.

\textbf{Decoding Information (DI).}
Suppose that
$P = \{V_1,V_2,\cdots,V_L\}$ is the optimal partition of $V$, and the partition-based expression of the ODSI is as follows:
\begin{equation}
    \begin{aligned}
        H_{2}^{p}(G)=& \sum_{j=1}^{L}
        \Big[
            \operatorname{vol}\left(V_{j}\right) /  \operatorname{vol}(G) \cdot \sum_{i=1}^{\lvert v_{j}\rvert} f\left(\operatorname{deg}\left(v_{i}^{j}\right) / \operatorname{vol}\left(V_{j}\right)\right)\\ &+f\left(\operatorname{vol}\left(V_{j}\right) / \operatorname{vol}(G)\right)
        \Big]
    \end{aligned}
  \label{eq:07}
\end{equation}
where $f(^*)$ is $f(x)=-x \log _{2} x$ and $f(0)=f(1)=0$. Because the partition $P$ is optimal, we have $H_2(G) = H_2^P(G)$. With the partition-based ODSI, the expression of DI can be rewritten as follows:
\begin{equation}
  \setlength{\abovedisplayskip}{1pt}
  \setlength{\belowdisplayskip}{1pt}
  \begin{aligned}
    D^{P}(G) &=H_{1}(G)-H_{2}(G) \\
    &=-\sum_{j=1}^{L} \frac{\operatorname{vol}\left(V_{j}\right)-\textsl g_{j}}{\operatorname{vol}(G)} \log _{2} \frac{\operatorname{vol}\left(V_{j}\right)}{\operatorname{vol}(G)}\\
    &=\sum_{j=1}^{L} D_{(j)}^{p}
    \end{aligned}
  \label{eq:08}
\end{equation}
where the partition $P$ that minimizes the TDSI is optimal.
$D^P(G)$ can be interpreted as the information saved by $P$ in graph $G$.
And DI builds a bridge between the ODSI and OTDSI of a graph through the partition $P$ which minimizes the TDSI.

Suppose that $P=\{V_1,V_2,\cdots,V_l\}$ is the partition of $V$, for nodes set
$V_i\in V$ and $V_j\in V$, we can define $\Delta_{(i, j)}^{p}(G)$ as follows:
\begin{equation}
  \setlength{\abovedisplayskip}{1pt}
  \setlength{\belowdisplayskip}{1pt}
  \begin{aligned}
    \Delta_{(i, j)}^{p}(G)= &D_{(i)}^{p}+D_{(j)}^{p}-D_{(i, j)}^{p}\\
    =&-\frac{\operatorname{vol}\left(V_{i}\right)-\textsl g_{i}}{\operatorname{vol}(G)} \log _{2} \frac{\operatorname{vol}\left(V_{i}\right)}{\operatorname{vol}(G)} \\
    &-\frac{\operatorname{vol}\left(V_{j}\right)-\textsl g_{j}}{\operatorname{vol}(G)} \log _{2} \frac{\operatorname{vol}\left(V_{j}\right)}{\operatorname{vol}(G)} \\
    & +\frac{\operatorname{vol}\left(V_{i j}\right)-\textsl g_{i j}}{\operatorname{vol}(G)} \log _{2} \frac{\operatorname{vol}\left(V_{i j}\right)}{\operatorname{vol}(G)}
    \end{aligned}
  \label{eq:09}
\end{equation}
where $V_i\cup V_j=V_{ij}$, $\textsl g_{ij}$ is the number of edges from $V_{ij}$ to nodes outside of $V_{ij}$.
\begin{theorem}\label{th:01}
    For node sets $V_i$ and $V_j$ ( $V_i\subset V$, $V_j\subset V$, $V_i\cap V_j=\varnothing$ ), if there is no edge between $V_i$ and $V_j$, $\Delta_{(i, j)}^{p}(G) \geq 0$ is satisfied.
\end{theorem}

\begin{proof}[Proof of Theorem~{\upshape\ref{th:01}}]

By definition of partition-based DI, for the partition $P$,
$D^{P}(G)=\sum_{j=1}^{L} D_{(j)}^{P}=-\sum_{j=1}^{L} \frac{\operatorname{vol}\left(V_{j}\right)-\textsl g_{j}}{\operatorname{vol}(G)} \log _{2} \frac{\operatorname{vol}\left(V_{j}\right)}{\operatorname{vol}(G)}$, thus  and
$D_{(j)}^{p}=-\frac{\operatorname{vol}\left(V_{j}\right)-\textsl g_{j}}{\operatorname{vol}(G)}$ $\log _{2} \frac{\operatorname{vol}\left(V_{j}\right)}{\operatorname{vol}(G)}$. If there is no edge between $V_i$ and $V_j$, then
$\operatorname{vol}\left(V_{i}\right)+\operatorname{vol}\left(V_{j}\right)=\operatorname{vol}\left(V_{i j}\right)$,
$\textsl g_i+\textsl g_j=\textsl g_{ij}$. To simplify the notations, we relabel
$\operatorname{vol}(V_i)\Rightarrow V_i$, and 
$\operatorname{vol}(G) \cdot \Delta_{(i, j)}^{P}(G)$ is indicated as
$\operatorname{vol}(G) \cdot \Delta_{(i, j)}^{p}(G)=\operatorname{vol}(G) \cdot\left(D_{(i)}^{p}+D_{(j)}^{p}-D_{(i, j)}^{p}\right)$, thus
$$
\begin{aligned}
  \operatorname{vol}(G) \cdot \Delta_{(i, j)}^{p} (G)=&\operatorname{vol}(G) \cdot\left(D_{(i)}^{p}+D_{(j)}^{p}-D_{(i, j)}^{p}\right) \\
  =&-\left(V_{i}-\textsl g_{i}\right) \log _{2} \frac{V_{i}}{\operatorname{vol}(G)} \\
  &-\left(V_{j}-\textsl g_{j}\right) \log _{2} \frac{V_{j}}{\operatorname{vol}(G)} \\
  &+\left(V_{i j}-\textsl g_{i j}\right) \log _{2} \frac{V_{i j}}{\operatorname{vol}(G)} \\
  =&-\left(V_{i}-\textsl g_{i}\right) \log _{2}\left(\frac{V_{i}}{V_{i}+V_{j}}\right)\\
  &-\left(V_{j}-\textsl g_{j}\right) \log _{2}\left(\frac{V_{j}}{V_{i}+V_{j}}\right) \\
  \geq & 0
  \end{aligned}
$$
Finally, $\Delta_{(i, j)}^{p}(G) \geq 0$ holds in the Theorem~{\upshape\ref{th:01}}.
\end{proof}

\begin{theorem}\label{th:02}
    Maximizing the decoding information is equal to minimizing the two-dimensional structural information entropy in the data space.
\end{theorem}

\begin{figure}[htbp]
  \centering
  \includegraphics[width=0.7\linewidth]{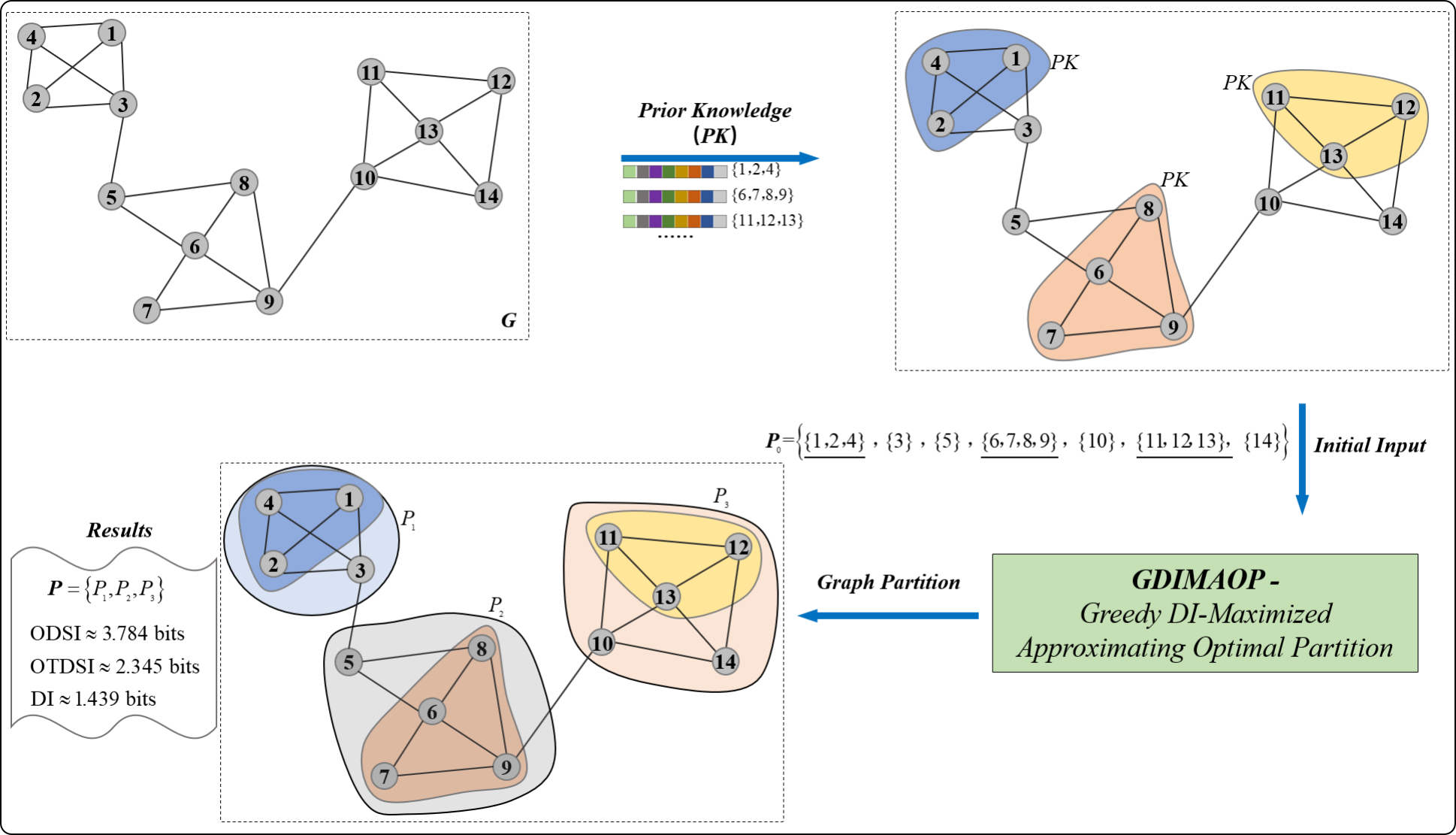}
  \caption{Workflow for partitioning graph based on the greedy DI-Maximized approximating optimal partition (GDIMAOP) with prior knowledge (PK).}  
  \label{fig:02}
\end{figure}

The DI-Maximized Partition Algorithm, detailed in this work, introduces the Greedy DI-Maximized Approximating Optimal Partition (GDIMAOP) method. GDIMAOP, a greedy algorithm, aims to approximate the optimal partition of a graph into distinct clusters. It involves three main steps: 1) Initial Partitioning: Singleton $P_i={v_i}$ is separated from $V$. 2) Minimizing $\Delta_{(i, j)}^{p}(G)$: Based on Theorem 1, this step finds the minimized value for all pairs $P_i$, $P_j$ until no pair satisfies $\Delta_{(i, j)}^{p}(P_i,P_j)\textgreater 0$. 3) Sorting Partition: The final step sorts the partition according to the node order in $V$.
GDIMAOP can incorporate partial hierarchical knowledge of nodes to improve its efficiency. An example is given where vertices in a graph $G$ are grouped into three clusters based on prior knowledge. This adaptation of GDIMAOP, which includes partial vertex knowledge, is referred to as PK-GDIMAOP. This method is particularly effective in enhancing the control and accuracy of the algorithm when such prior information is available.

\begin{algorithm}[htbp]
  \setlength{\abovedisplayskip}{-3pt}
  \setlength{\belowdisplayskip}{-3pt}
  \caption{\textbf{Greedy DI-Maximized Approximating Optimal Partition, GDIMAOP}}
  \label{algo:03}
  \KwIn{$G(V,E,W)$ - The extracted structure from data points where $|V|=n$, $|E|=m$, and $W\in R^{n\times n}$ is adjacency weights matrix.}
  \KwOut{$P=\{P_1, P_2,\cdots,P_L\}$ - Graph partition where $P_i\cap P_j = \varnothing$ and $\cup P_i = V$.}
  $\{P_1,P_2,\cdots,P_n\}\leftarrow CutVertex(V)$ where $P_i\cap P_j=\varnothing$, $\cup P_i=V$, and $|P_i|=1$.\\\label{algo03ln01}
  $\textbf{P}_0\leftarrow\{P_1,P_2,\cdots,P_n\}$.\\\label{algo03ln02}
  \Repeat{\label{algo03ln03}
    there is no $(P_i,P_j)$ such that $\Delta_{(i,j)}^{P}(P_i,P_j)\textgreater 0$
  }{
    $\delta_{min},P_{min}\leftarrow 0.0, \varnothing$.\\
    \For{$P_i$ in $P_0$}{
      ${P'}_0 \leftarrow P_0-P_i$.\\
      \For{$P_j$ in ${P'}_0$}{
        $\Delta_{(i,j)}^{P}\leftarrow\Delta_{(i,j)}^{P}(P_i,/ P_j)$.//Equation\eqref{eq:09}.\\
        \If{$\Delta_{(i,j)}^{P}\textless 0$ And $\Delta_{(i,j)}^{P}\textless \delta_{min}$}{
          $\delta_{min}\leftarrow\Delta_{(i,j)}^{P}$.\\
          $P_{min}\leftarrow P_j$.\\
        }
      }
      \uIf {$P_{min}\neq\varnothing$ And $P_{min}\in{P'}_0$} {
        ${P'}_0\leftarrow{P'}_0-P_{min}$.\\
        ${P'}_0\cdot Insert(\{P_i+P_j\})$.\\
        } \Else {
          ${P'}_0\cdot Insert(\{P_i\})$.\\
        } 
      $P_0\leftarrow{P'}_0$.\\
    }
  }\label{algo03ln22}
  $P\leftarrow SortByOrder(P_0)$.\\\label{algo03ln23}
  \Return{P}.
\end{algorithm}

\begin{algorithm}[htbp]
  \caption{\textbf{Clustering Algorithm for Maximum Decoding Information, CMDI}}
  \label{algo:04}
  \KwIn{$\textbf{X}\in R^{n\times m}$ - Data matrix with $n$ samples and $m$ features.\\
  \ \ \ \ \ \ \ \ \ \ \ \ $PK\in Z^n$- Prior knowledge list for $n$ data points.}
  \KwOut{$Y\in Z^n$ - Clustering results with $k$ labels.}
  $G\leftarrow ExtractingStructure(\textbf{X})$.//Extracting structure.\\\label{algo04ln01}
  $P\leftarrow GDIMAOP(G)$.// DI-Maximized partition process.\\\label{algo04ln02}
  or\\
  $P\leftarrow PK-GDIMAOP(G,PK)$.//Partition process with PK.\\\label{algo04ln04}
  $Y\leftarrow Mapping(P)$.//Mapping the partitions into $k$ clusters.\\\label{algo04ln05}
  \Return{Y}.
\end{algorithm}

\textbf{CMDI: Clustering Algorithm for Maximum Decoding Information}
CMDI, or Decoding Information Maximized Clustering Algorithm, expands on the structure extraction and DI-Maximized partition algorithms. It includes three main steps:
1) Graph Reconstruction: Using the structure extraction process to reconstruct the graph structure.
2) Partitioning: Depending on the availability of prior knowledge (PK), CMDI uses either the original Greedy DI-Maximized Approximating Optimal Partition (GDIMAOP) method for partitioning without PK, or the PK-GDIMAOP method when PK is available. The latter enhances clustering efficiency.
3) Mapping to Clusters: The partitions are then mapped into $k$ clusters, with $k$ representing the partition size.
CMDI operates in two modes: CMDI without PK, and CMDI-PK, which incorporates prior knowledge.

\textbf{Decoding Information Ratio (DI-R)}
DI-R is used to uniformly evaluate graph partitioning strategies across different graphs derived using various proximity metrics. It addresses the variations in the extracted graph structure resulting from different natural associations in the dataset. This measure allows for a standardized assessment of the efficiency and effectiveness of the partitioning strategy employed in the graph analysis.

\begin{equation}
  \setlength{\abovedisplayskip}{1pt}
  \setlength{\belowdisplayskip}{1pt}
  DI-R=\frac{D(G)}{H_{1}(G)}
  \label{eq:10}
\end{equation}
Based on \textit{definition 2.4}, the decoding information can be written as
$D(G)=H_{1}(G)-H_{2}(G)$. Furtherly, we could find the truth that 0.0 \textless DI-R \textless 1.0.

\begin{table*}[!htbp]
  \centering
  \caption{Experimental Datasets and Comments.}
  \label{tab:01}
  \begin{tabularx}{\linewidth}{@{}lX@{}}
    \toprule
  Name (Tag)            & Comments                                                                                                                               \\ 
  \midrule
  BJ\_6$^{th}$R\_Geo         & (Real dataset) Simplified graph topology with 115,101 nodes and 168,041 edges                                                        \\
  BJ\_Busstop           & (Real   dataset) 42,161 bus stops (1543 bus routes) of Beijing in 2013, the   coordinate of latitude and longitude with “GCS WGS 1984” \\
  BJ\_Parking           & (Real dataset) 5,881 parking places of Beijing in 2014, the coordinate of latitude and longitude with   “GCS WGS 1984”                 \\
  Ring of Cliques graph & (Synthetic dataset) This is   dataset contains 6 cliques that each clique includes 5 nodes                                             \\
  Gid graph             & (Synthetic dataset) With 5x6 grid-type nodes with   proper links                                                                       \\
  Scale-free graph      & (Synthetic dataset) This   graph is built based on Barabasi-Albert model with 30 nodes                                                \\ 
  \bottomrule
  \end{tabularx}
\end{table*}

{\footnotesize
\begin{table*}[p]
  \centering
  \caption{Graph HIM-Distances of the Extracted Graph Based on Proximity Metrics.}
  \label{tab:02}
  \begin{threeparttable}
  \begin{tabular}{cccccccccc}
    \toprule
  \multirow{2}{*}{Proximity Metrics} & \multicolumn{3}{c}{Ring of Cliques Graph}        & \multicolumn{3}{c}{Gid Graph}                    & \multicolumn{3}{c}{Scale-free Graph}             \\ \cline{2-10} 
                                     & KIM\tnote{$^a$}          & BM$^b$           & IGM$^c$          & KIM            & BM             & IGM            & KIM            & BM             & IGM            \\ 
                                     \midrule
  EUC                                & 0.463          & 0.334          & 0.572          & 0.6            & 0.703          & 0.664          & 0.633          & 0.856          & 0.564          \\
  MANH                               & 0.727          & 0.698          & 0.755          & 0.785          & 0.766          & 0.788          & 0.695          & 0.857          & 0.555          \\
  MINK($\lambda=3$)                                   & 0.67           & 0.5            & 0.511          & \textbf{0.321} & 0.589          & 0.601          & \textbf{0.324} & 0.798          & 0.724          \\
  CHEB                               & 0.814          & 0.814          & 0.814          & 0.832          & 0.832          & 0.832          & 0.863          & 0.863          & 0.863          \\
  CANB                               & 0.664          & 0.688          & 0.792          & 0.805          & 0.733          & 0.775          & 0.721          & \textbf{0.378} & 0.496          \\
  MAHA                               & 0.744          & 0.728          & 0.75           & 0.789          & 0.771          & 0.787          & 0.685          & 0.862          & 0.624          \\
  ANGU                               & 0.814          & 0.814          & 0.823          & 0.829          & 0.787          & 0.746          & 0.824          & 0.734          & 0.725          \\
  P-COR                              & 0.316          & \textbf{0.229} & \textbf{0.325} & 0.463          & \textbf{0.234} & 0.353          & 0.441          & 0.496          & \textbf{0.447} \\
  GAUS                               & 0.814          & 0.814          & 0.814          & 0.832          & 0.832          & 0.832          & 0.866          & 0.866          & 0.866          \\
  $\varepsilon$-NE                                   & \textbf{0.052} & 0.721          & 0.487          & 0.446          & 0.609          & \textbf{0.172} & 0.542          & 0.641          & 0.623          \\
  $k$-NN                                   & \textbf{0.075} & 0.593          & 0.346          & 0.429          & 0.629          & \textbf{0.172} & 0.497          & 0.648          & 0.618          \\
  \bottomrule
\end{tabular}
\begin{tablenotes}
  \footnotesize
  \item[{$^a$}] KIM: Kinetic Ising Model. $^b$ BM: Branching-Process Model. $^c$ IGM: Ising-Glauber Model.
\end{tablenotes}
\end{threeparttable}
\end{table*}
}

\section{Experiment Evaluation}\label{se:05}
In this section, we empirically assess the CMDI algorithm proposed in this study.
All experiments are executed on an Intel(R) Xeon(R) Silver server comprising 30 processors, each equipped with 10 CPU cores clocked at 2.40GHz and 64GB memory. 
The proposed algorithms are implemented using Python. 
% The source code, data, and other artefacts are accessible at attachment material \footnote{It will be public on GitHub later.}.
% % \url{https://github.com/jeshy/MaxDICluster.git}.

\subsection{Experimental Settings}\label{subse:0501}
Three real-world datasets and three synthetic datasets in our experiments are listed in Table \ref{tab:01}.

\textbf{Three Real-world Datasets:}
The \textbf{BJ\_$6^{th}$R\_Geo} dataset, derived from OpenStreetMap (OSM) \footnote{https://www.openstreetmap.org/}, features over 115,000 geographical locations of road intersections within Beijing's Sixth Ring, refined to exclude non-intersection nodes and dead ends, resulting in 115,101 nodes and 168,041 edges. \textbf{BJ\_Busstop}, sourced from Beijing City Lab (BCL) \footnote{https://www.beijingcitylab.com/}, includes 42,161 bus stops (1543 bus routes) in Beijing, compiled by Dr. Jianghao Wang in 2013 in the "GCS WGS 1984" coordinate system. \textbf{BJ\_Parking} contains all Beijing parking locations for 2014, also coordinated in "GCS WGS 1984".

\begin{figure}[htbp]
  \centering
  \includegraphics[width=0.7\linewidth]{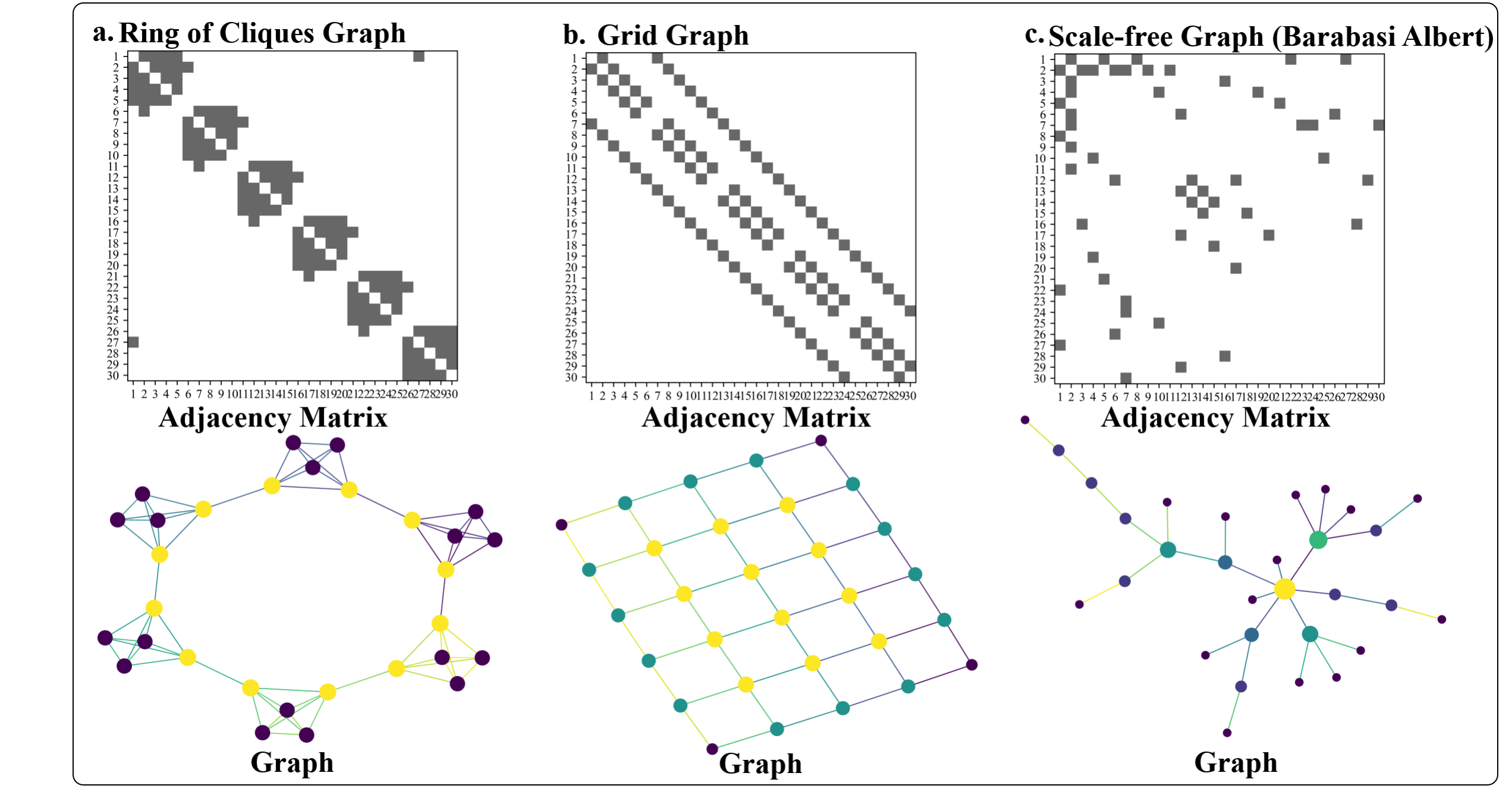}
  \caption{Illustration of the three different graphs used in the present paper. a. Ring of clique graph with six cliques. b. Grids graph. c. Scale-free graph with Barabasi-Albert model.}
  \label{fig:03}
\end{figure}

\begin{figure}[htbp]
  \centering
  \includegraphics[width=0.7\linewidth]{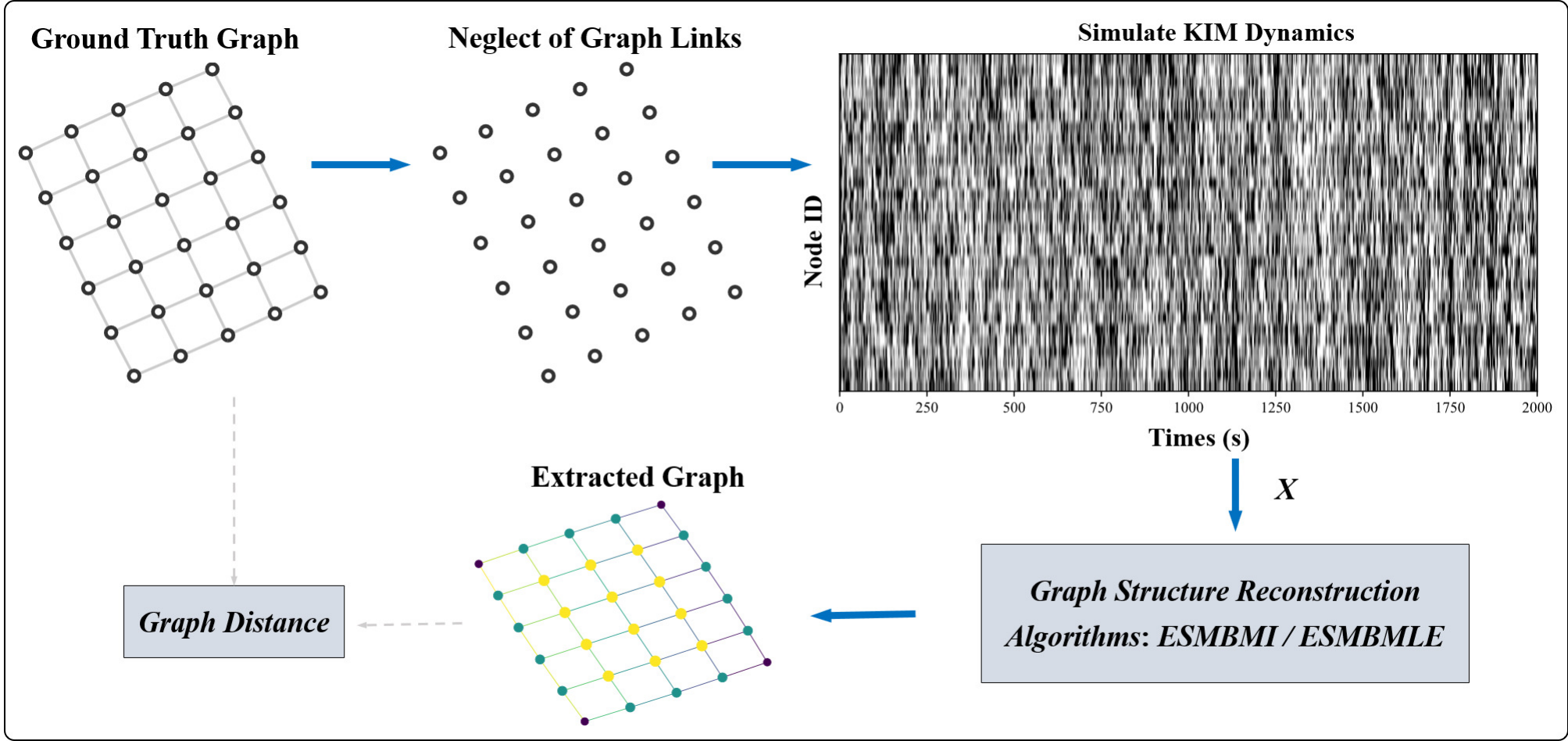}
  \caption{Workflow for generating synthetic time-series data from a given graph topology. 
  Initially, we select a graph of interest and then build the graph adjacency matrix $A$ whose elements are zeros (based on the ground truth graph, the link of the graph is neglected). Next, the Kinetic Ising model (KIM) is used to simulate the dynamical process of the graph, and then the synthetic time-series data $X$ is generated and the graph reconstruction method is taken to extract the graph structure (the structure reconstruction algorithms: ESMBMI and ESMBMLE). Finally, the graph distance is used to evaluate the similarity between the ground-truth graph and extracted graph.
  }
  \label{fig:04}
\end{figure}

\textbf{Three Synthetic Datasets:}
The evaluation of the ESMBMI and ESMBMLE methods is conducted on three distinct graph topologies using synthetic time-series data generated by the Kinetic Ising Model (KIM), with node connections disregarded for clarity. All graphs share the same dimension, $n=30$, and are based on $n \times n$ adjacency matrices. The three graph types used for comparison are:
\textbf{Ring of Cliques Graph:} Composed of interconnected cliques where each clique is a complete graph. In this setup, the graph contains 6 cliques, each with 5 nodes, as shown in Figure \ref{fig:03}a.
\textbf{Grid Graph:} Features a grid structure with nodes connected in a grid pattern. Both the adjacency matrices and the graph topology are demonstrated in Figure \ref{fig:03}b.
\textbf{Scale-Free Graph (Barabasi-Albert Model):} This graph follows a power-law degree distribution, created through preferential attachment from an initial set of $n_0$ nodes. New nodes are added with $m < n_0$ edges to the existing nodes based on a probability linked to each node's degree. The scale-free graph ultimately has $n-1$ edges. Figure \ref{fig:03}c shows a schematic of a scale-free graph with $n=30$, based on the Barabasi-Albert model.
These diverse graph structures allow for a comprehensive assessment of the ESMBMI and ESMBMLE methods' performance in reconstructing system structures from synthetic time-series data.

\textbf{Baselines and PKs:}
CMDI is benchmarked against MinBatch-K-Means and DBSCAN clustering methods. MinBatch-K-Means, a centroid-based model \cite{DBLP:conf/www/Sculley10}, and DBSCAN, a density-based model grouping densely packed data while marking sparse area points as outliers \cite{DBLP:journals/tods/SchubertSEKX17}, and SDCN \cite{bo2020structural}, a deep learning-based clustering approach that integrates GCN and DNN via a propagation operator and self-supervised modules for deep clustering in complex scenarios \cite{bo2020structural}, serve as baselines. Spectral clustering, similar to CMDI but more complex, is excluded due to server limitations. Experimental results show prior knowledge (PK) significantly boosts performance. Three PK types are employed: MinBatch-K-Means-based PK, DBSCAN-based PK, and Informap-based PK \cite{Rosvall1118}, each contributing to enhanced algorithm efficiency.

\textbf{Weight Threshold:}
In weighted graph extraction, a key rule is applied: edge weight is inversely proportional to node distance. A threshold parameter filters out minor weights for graph reliability and practicality. Edge weights are determined by the reciprocal of distance or similarity.

\textbf{Graph Distance:}
This paper employs the novel Hamming-Ipsen-Mikhailov(HIM) \cite{7344816} distance metric to quantitatively measure the dissimilarity between two graphs sharing the same vertices. The HIM distance metric is defined as follows:

\begin{scriptsize}
\begin{equation}
  \setlength{\abovedisplayskip}{1pt}
  \setlength{\belowdisplayskip}{1pt}
  \begin{aligned}
    H I M_{\xi}\left(G_{1}, G_{2}\right) &=\frac{1}{\sqrt{1+\xi}}\lvert\left(H\left(G_{1}, G_{2}\right), \sqrt{\xi} \cdot I M\left(G_{1}, G_{2}\right)\right)\rvert_{2} \\
    &=\frac{1}{\sqrt{1+\xi}} \sqrt{H^{2}\left(G_{1}, G_{2}\right)+\xi \cdot I M^{2}\left(G_{1}, G_{2}\right)}
  \end{aligned}
  \label{eq:11}
\end{equation}
\end{scriptsize}

where $G_x$ denotes the extracted graph, $H(\cdot)$ and $IM(\cdot)$ represent the Hamming distance metric and the Ipsen-Mikhailov \cite{MA02} distance respectively,
$\xi$ denotes the impact factor. Apparently,
$HIM_0(\cdot)=H(\cdot)$ and
$\lim _{\xi \Rightarrow+\infty} H I M_{\xi}(\cdot)=I M(\cdot)$. The distance metric
$H I M_{\xi}\left(G_{1}, G_{2}\right)$ is bounded in the interval [0, 1] and all distances will be nonzero for non-identical isomorphic graphs.

\subsection{HIM Distance and DI in the Synthetic Dataset}
This subsection primarily assesses the performance of graph extraction methods on synthetic datasets in terms of HIM distance and DI. The Kinetic Ising model (KIM) is employed to simulate the dynamical process of obtaining synthetic time-series data, capturing the graph's dynamic evolution during data generation. For comparison purposes, the length of the time-series data is standardized to 2001 seconds. Figures \ref{fig:05} and \ref{fig:06} present the corresponding results on synthetic data. Notably, the HIM distances of the structure extraction method based on maximum likelihood estimation are consistently the smallest.
An intriguing finding is that the graph HIM distances for the correlation matrix, graphical lasso, and mutual information matrix structure extraction methods exhibit no variance in realizing the ring of cliques graph, grid graph, and scale-free graph. Figure \ref{fig:06} illustrates the DI with different structure extraction methods across three graph types, and once again, the method based on maximum likelihood estimation demonstrates superior performance.
In summary, the results depicted in Figures \ref{fig:05} and \ref{fig:06} highlight that the maximum likelihood estimation method excels at capturing the inherent natural correlations between data points in synthetic datasets. Therefore, for time-series datasets, the maximum likelihood estimation-based method proves to be the most effective approach for accurately deriving a graph from a distribution that incorporates edge weight information and degree-dependent distributions.

\subsection{ HIM Distance and DI of the Extracted Graph Based on Different Proximity Metrics}
Graph reconstruction is essential for clustering methods based on graph models, where local neighbourhood relationships between data points are modelled. Consequently, various proximity measurements quantify the distance or closeness between two data objects, and the distance may vary across these metrics. In Subsection \ref{subse:0501}, we evaluated common graph reconstruction methods. In this subsection, we similarly assessed different proximity measurement methods for graph extraction.

Tables \ref{tab:02} and \ref{tab:03} present the graph HIM distances and decoding information obtained from a synthetic dataset using different proximity measurement methods. To facilitate comparison, three dynamical process models were employed to simulate time series on the synthetic network. Specifically, the simulation models include the Kinetic Ising Model \cite{DBLP:journals/amc/FonsecaKMM15} (KIM), Branching-Process Model \cite{DBLP:conf/dac/SastryM91} (BM), and Ising-Glauber Model \cite{MVG98} (IGM). The length of the time-series data is standardized to 2001 seconds for all simulation models, with the original synthetic graph serving as the ground truth. Additionally, we set an appropriate threshold to enhance efficiency.

Table \ref{tab:02} shows for the ring of cliques graph, $\varepsilon$-NE and $k$-NN are most effective in the KIM model (HIM distances 0.052, 0.075), while P-COR leads in the BM and IGM models (HIM distances 0.229, 0.325). For the grid graph, MINK($\lambda=3$) excels in the KIM model (HIM distance 0.321), P-COR in the BM model (HIM distance 0.234), and both $\varepsilon$-NE and $k$-NN in the IGM model (HIM distance 0.172). In the scale-free graph, MINK($\lambda=3$) is best for KIM (HIM distance 0.324), CANB for BM (HIM distance 0.378), and P-COR for IGM (HIM distance 0.447).

Table \ref{tab:03} summarizes the effectiveness of proximity metrics in various models and graphs. For the ring of cliques graph in the KIM model, $\varepsilon$-NE and $k$-NN lead (DI: 2.543, 2.289 bits), P-COR tops in the BM model (DI: 1.083 bits), and $k$-NN in IGM (DI: 1.211 bits). In the grid graph, MINK($\lambda=3$) is optimal for KIM (DI: 1.494 bits), P-COR for BM (DI: 1.021 bits), and both $\varepsilon$-NE and $k$-NN for IGM (DI: 1.601, 1.738 bits). For the scale-free graph, $\varepsilon$-NE and $k$-NN are best in KIM (DI: 1.131, 1.443 bits), CANB in BM (DI: 1.687 bits), and CANB in IGM (DI: 1.064 bits). Overall, $\varepsilon$-NE and $k$-NN are highly effective for time-series datasets as shown in Tables \ref{tab:02} and \ref{tab:03}.

{\footnotesize
\begin{table*}[p]
  \centering
  \caption{The Decoding Information (DI) of the Extracted Graph Based on Different Proximity Metrics.}
  \label{tab:03}
  \begin{threeparttable}
  \begin{tabular}{ccccccccccccc}

    \toprule
    \multirow{2}{*}{Proximity Metrics} & \multicolumn{4}{c}{Ring of Cliques Graph}      & \multicolumn{4}{c}{Gid Graph}                            & \multicolumn{4}{c}{Scale-free Graph}                     \\\cline{2-13}
    & GT\tnote{$^a$}   & KIM$^b$            & BM$^c$             & IGM$^d$   & GT    & KIM            & BM             & IGM            & GT    & KIM            & BM             & IGM            \\
    \midrule
EUC                                & 2.35 & 0.993          & 0.996          & 0.706 & 1.567 & 0.656          & 0.569          & 0.599          & 1.984 & 0.804          & 0.49           & 0.942          \\
MANH                               & 2.35 & 0.509          & 0.504          & 0.501 & 1.567 & 0.495          & 0.495          & 0.501          & 2.194 & 0.641          & 0.481          & 0.71           \\
MINK($\lambda=3$)                               & 2.35 & 0.541          & 0.66           & 0.864 & 1.567 & \textbf{1.494} & 0.568          & 0.721          & 2.08  & \textbf{1.443} & 0.529          & 0.602          \\
CHEB                               & 2.35 & 0.483          & 0.483          & 0.483 & 1.567 & 0.483          & 0.483          & 0.483          & 2.125 & 0.483          & 0.483          & 0.483          \\
CANB                               & 2.35 & 0.599          & 0.543          & 0.479 & 1.567 & 0.478          & 0.538          & 0.509          & 2.016 & 0.565          & \textbf{1.687} & \textbf{1.064} \\
MAHA                               & 2.35 & 0.53           & 0.521          & 0.506 & 1.567 & 0.5            & 0.496          & 0.481          & 2.194 & 0.647          & 0.484          & 0.668          \\
ANGU                               & 2.35 & 0.483          & 0.483          & 0.525 & 1.567 & 0.485          & 0.472          & 0.551          & 2.116 & 0.481          & 0.497          & 0.502          \\
P-COR                              & 2.35 & 0.82           & \textbf{1.083} & 0.85  & 1.567 & 0.778          & \textbf{1.021} & 0.871          & 2.255 & 0.874          & 0.884          & 0.812          \\
GAUS                               & 2.35 & 0.483          & 0.483          & 0.483 & 1.567 & 0.483          & 0.483          & 0.483          & 1.956 & 0.483          & 0.483          & 0.483          \\
$\varepsilon$-NE                                & 2.35 & \textbf{2.543} & 0.497          & 0.871 & 1.567 & 0.988          & 0.589          & \textbf{1.601} & 2.107 & 1.131          & 0.727          & 0.734          \\
$k$-NN                                & 2.35 & \textbf{2.289} & 0.661          & 1.211 & 1.567 & 1.003          & 0.554          & \textbf{1.738} & 2.224 & 0.897          & 0.644          & 0.822  \\
\bottomrule       
\end{tabular}

  \begin{tablenotes}
    \footnotesize
    \item[{$^a$}] GT: Ground Truth Graph. $^b$ KIM: Kinetic Ising Model. $^c$ BM: Branching-Process Model. $^d$ IGM: Ising-Glauber Model. 
  \end{tablenotes}
  \end{threeparttable}
\end{table*}
}

\begin{figure*}
  \centering
  \newlength{\subfigheight}
  \setlength{\subfigheight}{4cm} 
  \begin{minipage}[b]{.19\textwidth}
    \centering
    \includegraphics[width=\linewidth]{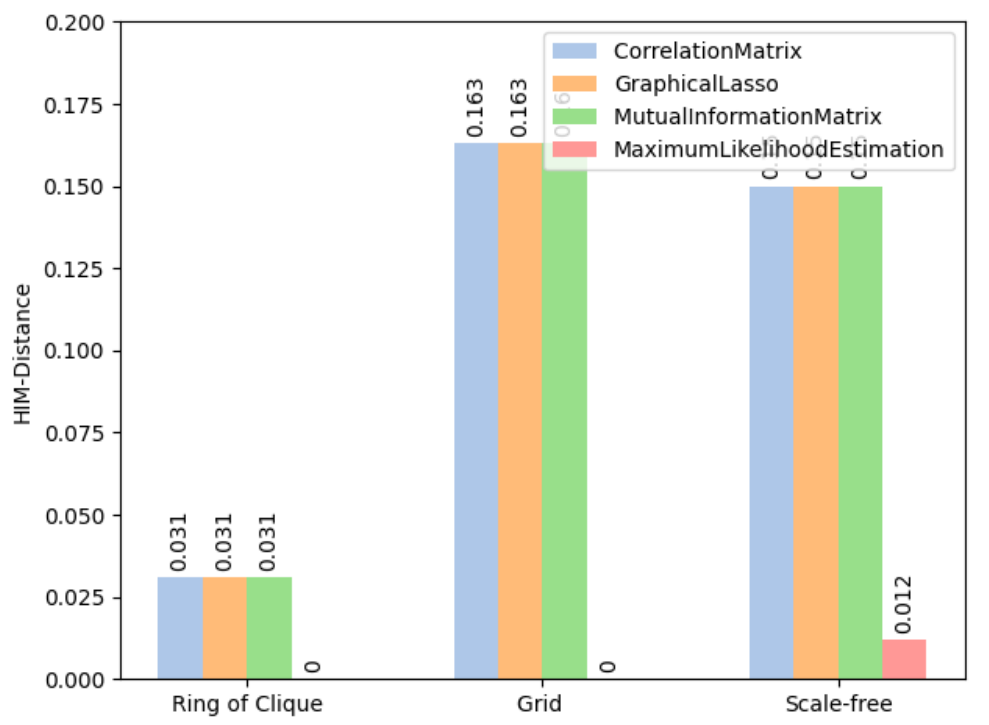}
    \caption{Comparison of HIM-Distances.}
    \label{fig:05}
  \end{minipage}
  \begin{minipage}[b]{.19\textwidth}
    \centering
    \includegraphics[width=\linewidth]{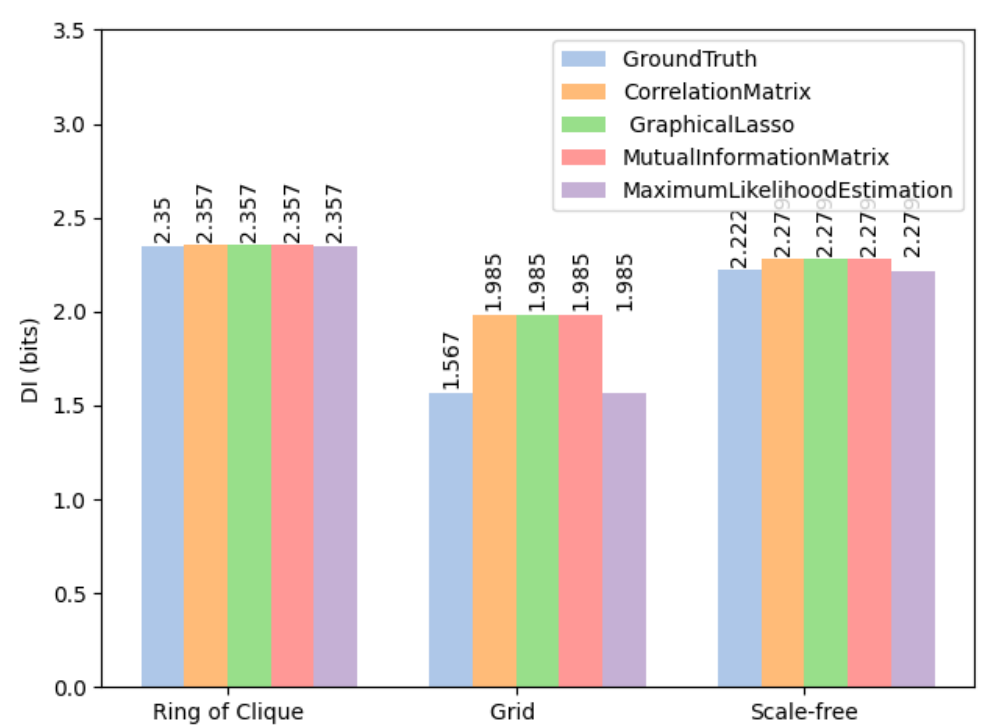}
    \caption{Comparison of DI.}
  \label{fig:06}
  \end{minipage}
  \begin{minipage}[b]{.19\textwidth}
    \centering
    \includegraphics[width=\linewidth]{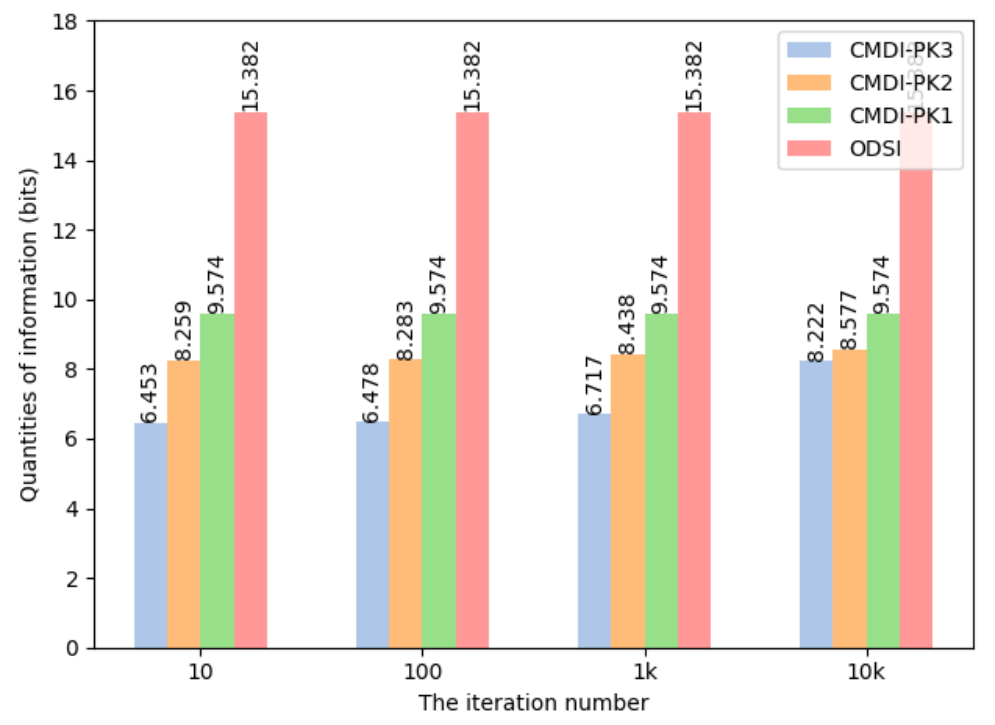}
    \caption{DI of GDIMAOP in BJ\_6$^{th}$R\_Geo.}
  \label{fig:08}
  \end{minipage}
  \begin{minipage}[b]{.19\textwidth}
    \centering
    \includegraphics[width=\linewidth]{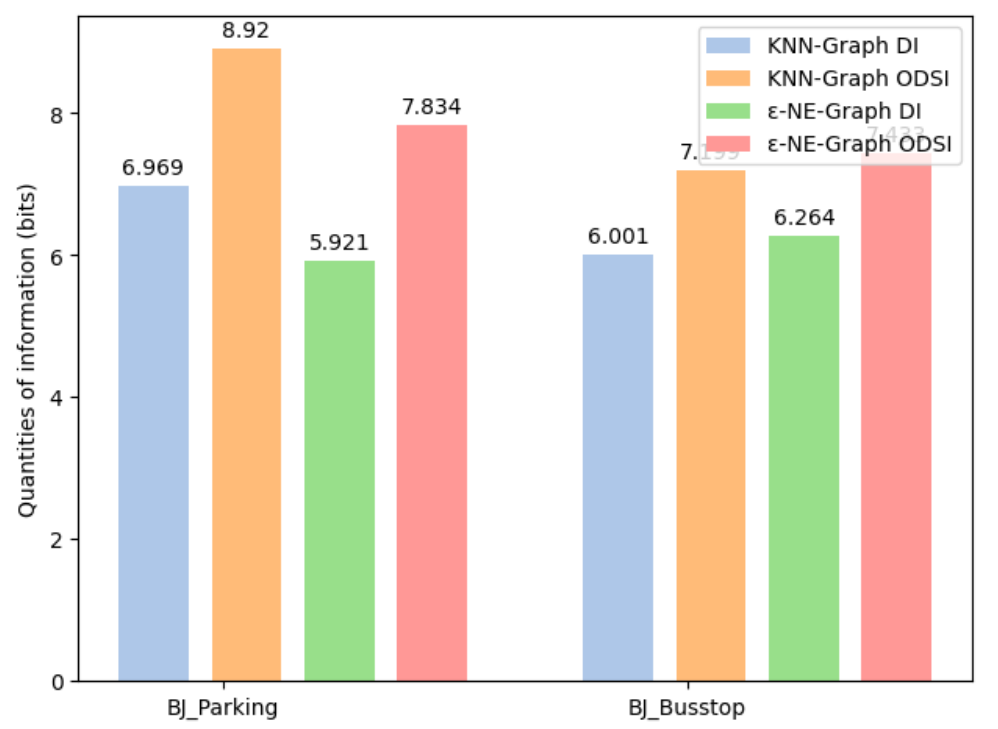}
    \caption{Results with PK1 Initialization.}
  \label{fig:09}
  \end{minipage}
  \begin{minipage}[b]{.19\textwidth}
    \centering
    \includegraphics[width=\linewidth]{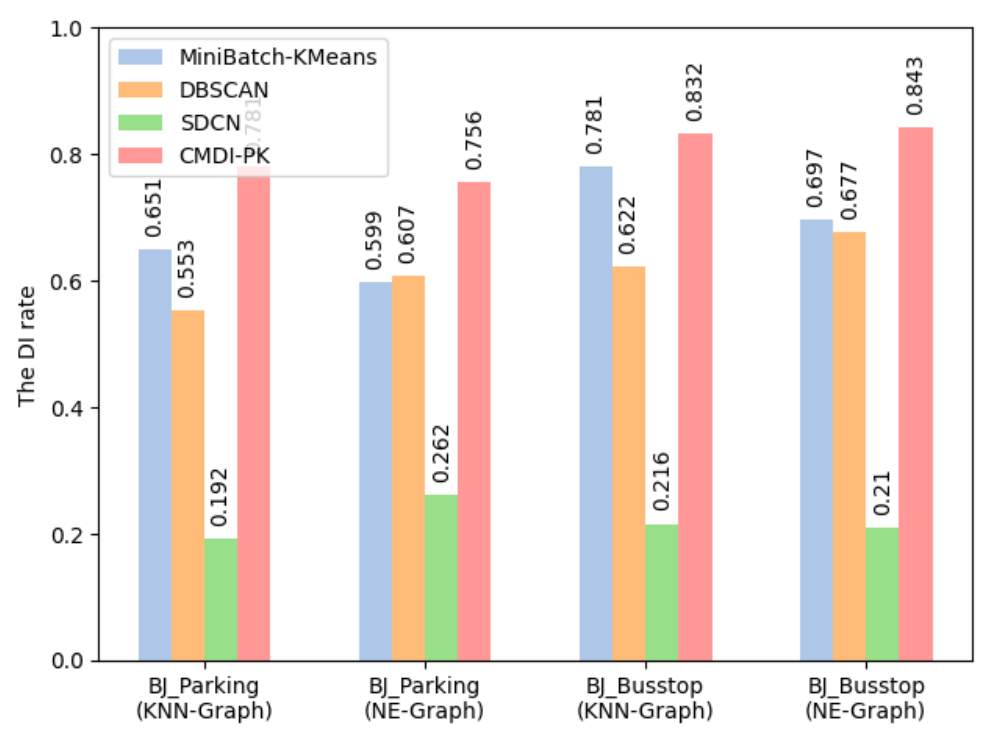}
    \caption{Comparison of Different Clustering Methods’ DI-R.}
  \label{fig:10}
  \end{minipage}
  \label{fig:test}
\end{figure*}

\begin{figure}[htbp]
  \centering
  \includegraphics[width=0.7\linewidth]{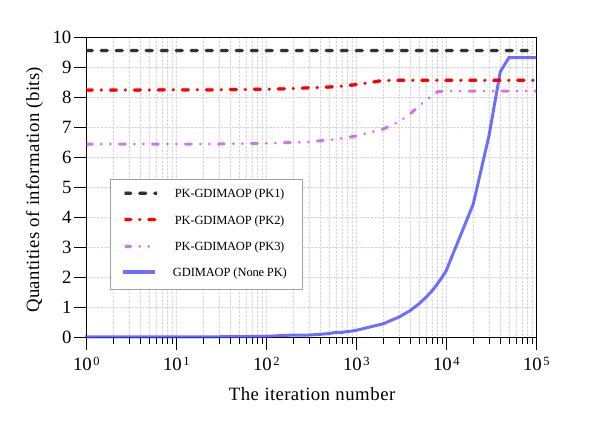}
  \caption{The Decoding Information (DI) State Updation of GDIMAOP with Different Prior Knowledge Initialization.} 
  \label{fig:07}
\end{figure}

\subsection{Decoding Information Quantities of GDIMAOP with Different PKs}
This subsection examines GDIMAOP's decoding information with different prior knowledge (PK) setups, as shown in Figure \ref{fig:07}. PK1 (Informap), PK2 (MinBatch-K-Means), and PK3 (DBSCAN) are used as initial partitions for GDIMAOP.

Figure \ref{fig:07} shows that using prior knowledge, especially PK1, significantly improves GDIMAOP's performance, maintains stable decoded information throughout iterations, and leads to faster convergence compared to GDIMAOP without prior knowledge. This suggests the effectiveness of integrating prior knowledge in graph vertex partitioning, a finding applicable to the CMDI algorithm where PK1 yields the most improvement, with CMDI-PK outperforming and converging faster than the original CMDI.

Evaluations on the BJ\_6$^{th}$R\_Geo dataset, as shown in Figure \ref{fig:08}, reveal that CMDI-PK1, using Informap as prior knowledge, achieves the highest decoding information, effectively decoding this real-life data.

\subsection{CMDI Performance on Geo-data (BJ\_Busstop and BJ\_Parking)}
This section evaluates CMDI's performance on geospatial datasets BJ\_Busstop and BJ\_Parking, as shown in Figure \ref{fig:09}. 
For BJ\_Parking, using $k$-NN, the one-dimensional structure entropy is 8.920 bits with a decoding information of 6.969 bits; $\varepsilon$-NE yields 7.834 bits and 5.921 bits, respectively. In BJ\_Busstop, $k$-NN results in a structural entropy of 7.199 bits and decoding information of 6.001 bits, while $\varepsilon$-NE shows 7.433 bits and 6.264 bits, respectively.

Figure \ref{fig:10} shows DI-R values for CMDI: 0.78 ($k$-NN, BJ\_Parking), 0.76 ($\varepsilon$-NE, BJ\_Parking), 0.83 ($k$-NN, BJ\_Busstop), and 0.84 ($\varepsilon$-NE, BJ\_Busstop). MinBatch-K-Means ranks second on average, while CMDI outperforms MinBatch-K-Means, DBSCAN, and SDCN in DI-R.

\section{CONCLUSIONS}\label{se:06}
This paper presents an innovative clustering methodology designed to optimize decoding information. Demonstrated through comprehensive experiments on three real-world datasets (BJ\_6$^{th}$R\_Geo, BJ\_Busstop, and BJ\_Parking), CMDI consistently surpasses traditional methods in DI-R efficiency. Significantly, CMDI marks a pioneering effort in seamlessly blending two-dimensional structural information theory with data clustering. This integration heralds new avenues for advanced applications in data mining, setting a precedent for future explorations in the field.

\bibliographystyle{elsarticle-num}
\bibliography{mybib}

\begin{thebibliography}{10}
\expandafter\ifx\csname url\endcsname\relax
  \def\url#1{\texttt{#1}}\fi
\expandafter\ifx\csname urlprefix\endcsname\relax\def\urlprefix{URL }\fi
\expandafter\ifx\csname href\endcsname\relax
  \def\href#1#2{#2} \def\path#1{#1}\fi

\bibitem{DBLP:conf/kdd/MonathDGZAMMNTT21}
N.~Monath, K.~A. Dubey, G.~Guruganesh, M.~Zaheer, A.~Ahmed, A.~McCallum, G.~Mergen, M.~Najork, M.~Terzihan, B.~Tjanaka, Y.~Wang, Y.~Wu, Scalable hierarchical agglomerative clustering, in: F.~Zhu, B.~C. Ooi, C.~Miao (Eds.), {KDD} '21: The 27th {ACM} {SIGKDD} Conference on Knowledge Discovery and Data Mining, Virtual Event, Singapore, August 14-18, 2021, {ACM}, pp. 1245--1255.

\bibitem{chen2023optimization}
S.~Chen, Y.~Tian, R.~Huang, Z.~Dai, Optimization simulation of big data analysis model based on k-means algorithm, in: 2023 International Conference on Networking, Informatics and Computing (ICNETIC), IEEE, 2023, pp. 764--768.

\bibitem{DBLP:journals/pr/ChenZBWCD21}
Y.~Chen, L.~Zhou, N.~Bouguila, C.~Wang, Y.~Chen, J.~Du, {BLOCK-DBSCAN:} fast clustering for large scale data, Pattern Recognit. 109 (2021) 107624.

\bibitem{DBLP:journals/jcst/WangZCQZLZZ21}
L.~Wang, H.~Zhang, H.~Chang, Q.~Qin, B.~Zhang, X.~Li, T.~Zhao, T.~Zhang, Gaebic: {A} novel biclustering analysis method for mirna-targeted gene data based on graph autoencoder, J. Comput. Sci. Technol. 36~(2) (2021) 299--309.

\bibitem{ahmed2021remaxint}
Z.~Ahmed, A.~Cassese, G.~van Breukelen, J.~Schepers, Remaxint: a two-mode clustering-based method for statistical inference on two-way interaction, Advances in Data Analysis and Classification 15~(4) (2021) 987--1013.

\bibitem{DBLP:journals/sac/AffeldtLN21}
S.~Affeldt, L.~Labiod, M.~Nadif, Regularized bi-directional co-clustering, Stat. Comput. 31~(3) (2021) 32.

\bibitem{DBLP:journals/isci/YangDLLGUH21}
Y.~Yang, S.~Deng, J.~Lu, Y.~Li, Z.~Gong, L.~H. U, Z.~Hao, Graphlshc: Towards large scale spectral hypergraph clustering, Inf. Sci. 544 (2021) 117--134.

\bibitem{DBLP:journals/concurrency/MaLZLC21}
S.~Ma, Y.~Liu, Q.~Zheng, Y.~Li, Z.~Cui, Multiview spectral clustering via complementary information, Concurr. Comput. Pract. Exp. 33~(15) (2021).

\bibitem{jiang2022flow}
S.~Jiang, Q.~Wang, C.~Wang, K.~Liu, S.~Ning, X.~Xu, Flow prediction via multi-view spatial-temporal graph neural network, in: International Conference on Data Mining and Big Data, Springer, 2022, pp. 77--92.

\bibitem{DBLP:journals/kbs/NguyenASCA20}
J.~Nguyen, A.~Armisen, G.~S{\'{a}}nchez{-}Hern{\'{a}}ndez, M.~Casabay{\'{o}}, N.~Agell, An owa-based hierarchical clustering approach to understanding users' lifestyles, Knowl. Based Syst. 190 (2020) 105308.

\bibitem{bezdan2021hybrid}
T.~Bezdan, C.~Stoean, A.~A. Naamany, N.~Bacanin, T.~A. Rashid, M.~Zivkovic, K.~Venkatachalam, Hybrid fruit-fly optimization algorithm with k-means for text document clustering, Mathematics 9~(16) (2021) 1929.

\bibitem{an2023strp}
X.~An, Z.~Wang, D.~Wang, S.~Liu, C.~Jin, X.~Xu, J.~Cao, Strp-dbscan: A parallel dbscan algorithm based on spatial-temporal random partitioning for clustering trajectory data, Applied Sciences 13~(20) (2023) 11122.

\bibitem{DBLP:conf/aaai/SunCWLF20}
G.~Sun, Y.~Cong, Q.~Wang, J.~Li, Y.~Fu, Lifelong spectral clustering, in: The Thirty-Fourth {AAAI} Conference on Artificial Intelligence, {AAAI} 2020, The Thirty-Second Innovative Applications of Artificial Intelligence Conference, {IAAI} 2020, The Tenth {AAAI} Symposium on Educational Advances in Artificial Intelligence, {EAAI} 2020, New York, NY, USA, February 7-12, 2020, {AAAI} Press, pp. 5867--5874.

\bibitem{DBLP:journals/infsof/WeiYGL21}
C.~Wei, X.~Yao, D.~Gong, H.~Liu, Spectral clustering based mutant reduction for mutation testing, Inf. Softw. Technol. 132 (2021) 106502.

\bibitem{zhang2023single}
X.~Zhang, Y.~Xu, X.~Zhang, Single-channel speech separation algorithm combining attention mechanism and clustering algorithm, in: Fourth International Conference on Signal Processing and Computer Science (SPCS 2023), Vol. 12970, SPIE, 2023, pp. 950--954.

\bibitem{DBLP:conf/icassp/LiNL18}
J.~Li, F.~Nie, X.~Li, Directly solving the original ratiocut problem for effective data clustering, in: 2018 {IEEE} International Conference on Acoustics, Speech and Signal Processing, {ICASSP} 2018, Calgary, AB, Canada, April 15-20, 2018, {IEEE}, pp. 2306--2310.

\bibitem{DBLP:conf/iccv/ChewC15}
S.~E. Chew, N.~D. Cahill, Semi-supervised normalized cuts for image segmentation, in: 2015 {IEEE} International Conference on Computer Vision, {ICCV} 2015, Santiago, Chile, December 7-13, 2015, {IEEE} Computer Society, pp. 1716--1723.

\bibitem{DBLP:journals/tit/LiP16}
A.~Li, Y.~Pan, Structural information and dynamical complexity of networks, {IEEE} Trans. Inf. Theory 62~(6) (2016) 3290--3339.

\bibitem{DBLP:journals/tit/LiYXWHWDXZ18}
A.~Li, X.~Yin, B.~Xu, D.~Wang, J.~Han, Y.~Wei, Y.~Deng, Y.~Xiong, Z.~Zhang, Decoding topologically associating domains with ultra-low resolution hi-c data by graph structural entropy, Nat. Commun. 9~(3265) (2018).

\bibitem{xu2023optimal}
X.~Xu, Z.~Ding, Y.~Wu, J.~Yan, S.~Jiang, Q.~Cui, Optimal node embedding dimension selection using overall entropy, in: International Conference on Artificial Neural Networks, Springer, 2023, pp. 114--127.

\bibitem{DBLP:journals/tit/LiA21b}
C.~T. Li, V.~Anantharam, One-shot variable-length secret key agreement approaching mutual information, {IEEE} Trans. Inf. Theory 67~(8) (2021) 5509--5525.

\bibitem{DBLP:journals/csda/YuTR20}
T.~Yu, H.~Tsai, H.~Rachinger, Approximate maximum likelihood estimation of a threshold diffusion process, Comput. Stat. Data Anal. 142 (2020).

\bibitem{DBLP:journals/amc/FonsecaKMM15}
C.~M. da~Fonseca, S.~Kouachi, D.~A. Mazilu, I.~Mazilu, A multi-temperature kinetic ising model and the eigenvalues of some perturbed jacobi matrices, Appl. Math. Comput. 259 (2015) 205--211.

\bibitem{DBLP:journals/ijisp/MahantaK18}
H.~J. Mahanta, A.~K. Khan, Improving power analysis peak distribution using canberra distance to address ghost peak problem, Int. J. Inf. Secur. Priv. 12~(3) (2018) 27--41.

\bibitem{DBLP:journals/ijcse/ShahZ21}
D.~Shah, T.~Zaveri, Hyperspectral endmember extraction using pearson's correlation coefficient, Int. J. Comput. Sci. Eng. 24~(1) (2021) 89--97.

\bibitem{DBLP:journals/cin/ZhangLW20}
Z.~Zhang, X.~Liu, L.~Wang, Spectral clustering algorithm based on improved gaussian kernel function and beetle antennae search with damping factor, Comput. Intell. Neurosci. 2020 (2020) 1648573:1--1648573:9.

\bibitem{DBLP:conf/www/Sculley10}
D.~Sculley, Web-scale k-means clustering, in: M.~Rappa, P.~Jones, J.~Freire, S.~Chakrabarti (Eds.), Proceedings of the 19th International Conference on World Wide Web, {WWW} 2010, Raleigh, North Carolina, USA, April 26-30, 2010, {ACM}, pp. 1177--1178.

\bibitem{DBLP:journals/tods/SchubertSEKX17}
E.~Schubert, J.~Sander, M.~Ester, H.~Kriegel, X.~Xu, {DBSCAN} revisited, revisited: Why and how you should (still) use {DBSCAN}, {ACM} Trans. Database Syst. 42~(3) (2017) 19:1--19:21.

\bibitem{bo2020structural}
D.~Bo, X.~Wang, C.~Shi, M.~Zhu, E.~Lu, P.~Cui, Structural deep clustering network, in: Proceedings of the web conference 2020, 2020, pp. 1400--1410.

\bibitem{Rosvall1118}
M.~Rosvall, C.~T. Bergstrom, Maps of random walks on complex networks reveal community structure, Proceedings of the National Academy of Sciences 105~(4) (2008) 1118--1123.
\newblock \href {http://arxiv.org/abs/https://www.pnas.org/content/105/4/1118.full.pdf} {\path{arXiv:https://www.pnas.org/content/105/4/1118.full.pdf}}.

\bibitem{7344816}
G.~Jurman, R.~Visintainer, M.~Filosi, S.~Riccadonna, C.~Furlanello, The him glocal metric and kernel for network comparison and classification, in: 2015 IEEE International Conference on DSAA, 2015, pp. 1--10.
\newblock \href {https://doi.org/10.1109/DSAA.2015.7344816} {\path{doi:10.1109/DSAA.2015.7344816}}.

\bibitem{MA02}
M.~Ipsen, A.~S. Mikhailov, Evolutionary reconstruction of networks, Physical Review E~(66) (2002) 046109--1--046109–4.

\bibitem{DBLP:conf/dac/SastryM91}
S.~Sastry, A.~Majumdar, A branching process model for observability analysis of combinational circuits, in: A.~R. Newton (Ed.), Proceedings of the 28th Design Automation Conference, San Francisco, California, USA, June 17-21, 1991, {ACM}, pp. 452--457.

\bibitem{MVG98}
V.~A. Murav’ev, V.~M. Vorob’ev, A.~S. Garevskii, Microkinetics of the ising-glauber model in the binary approximation, Physics of the Solid State 40~(3) (1998) 477–--481.

\end{thebibliography}
\end{document}